\documentclass[10pt,twocolumn,letterpaper]{article}

\usepackage{cvpr}
\usepackage{times}
\usepackage{epsfig}
\usepackage{graphicx}
\usepackage{amsmath}
\usepackage{amssymb}
\usepackage{dsfont}
\usepackage{bm}
\usepackage{amsthm}
\usepackage{booktabs}
\usepackage[detect-all,per-mode=symbol]{siunitx}

\newcommand{\vect}[1]{\bm{\mathrm{#1}}} 
\newcommand{\vectwo}[1]{\bm{\mathrm{#1}}} 
\newcommand{\mat}[1]{#1}

\newcommand{\liegroup}{\mathcal{L}}
\newcommand{\liealg}{\mathcal{A}}
\newcommand{\SO}{\mathit{SO}}
\newcommand{\SE}{\mathit{SE}}
\newcommand{\so}{\mathfrak{so}}

\newcommand{\liedim}{d}
\DeclareMathOperator{\Exp}{Exp}
\DeclareMathOperator{\Log}{Log}
\newcommand{\Jr}{\mat{J}_\text{r}} 
\newcommand{\dif}{\mathrm{d}}
\newcommand{\Adj}[1]{\operatorname{Adj}_{#1}}
\newcommand{\R}{\mathbb{R}} 
\newcommand{\id}{\mathds{1}}
\newcommand{\norm}[1]{\|#1\|}
\newcommand{\ddj}[1]{\frac{\partial #1}{\partial\vect{d}_j}}

\newcommand{\blue}[1]{{\color{blue}#1}}

\newtheorem{theorem}{Theorem}[section]

\newtheorem{definition}[theorem]{Definition}

\newcommand{\labeleq}[1]{\ensuremath{\stackrel{\text{#1}}{=}}}

\cvprfinalcopy 

\def\cvprPaperID{7746} 

\ifcvprfinal\fi
\begin{document}

\title{Efficient Derivative Computation for Cumulative B-Splines on Lie Groups}

\author{Christiane Sommer$^{*}$ \quad Vladyslav Usenko$^{*}$ \quad David Schubert \quad Nikolaus Demmel \quad Daniel Cremers\\
Technical University of Munich\\
}

\maketitle
\thispagestyle{empty}

\ifcvprfinal
\let\thefootnote\relax\footnote{$^{*}$ These authors contributed equally. \\ This work was supported by the ERC Consolidator Grant ``3D Reloaded''.}
\fi

\begin{abstract}
Continuous-time trajectory representation has recently gained popularity for tasks where the fusion of high-frame-rate sensors and multiple unsynchronized devices is required.
Lie group cumulative B-splines are a popular way of representing continuous trajectories without singularities. They have been used in near real-time SLAM and odometry systems with IMU, LiDAR, regular, RGB-D and event cameras, as well as for offline calibration.

These applications require efficient computation of time derivatives (velocity, acceleration), but all prior works rely on a computationally suboptimal formulation.
In this work we present an alternative derivation of time derivatives based on recurrence relations that needs $\mathcal{O}(k)$ instead of $\mathcal{O}(k^2)$ matrix operations (for a spline of order $k$) and results in simple and elegant expressions.
While producing the same result, the proposed approach significantly speeds up the trajectory optimization and allows for computing simple analytic derivatives with respect to spline knots.
The results presented in this paper pave the way for incorporating continuous-time trajectory representations into more applications where real-time performance is required.
\end{abstract}


\begin{figure}
    \centering
    \includegraphics[width=\linewidth]{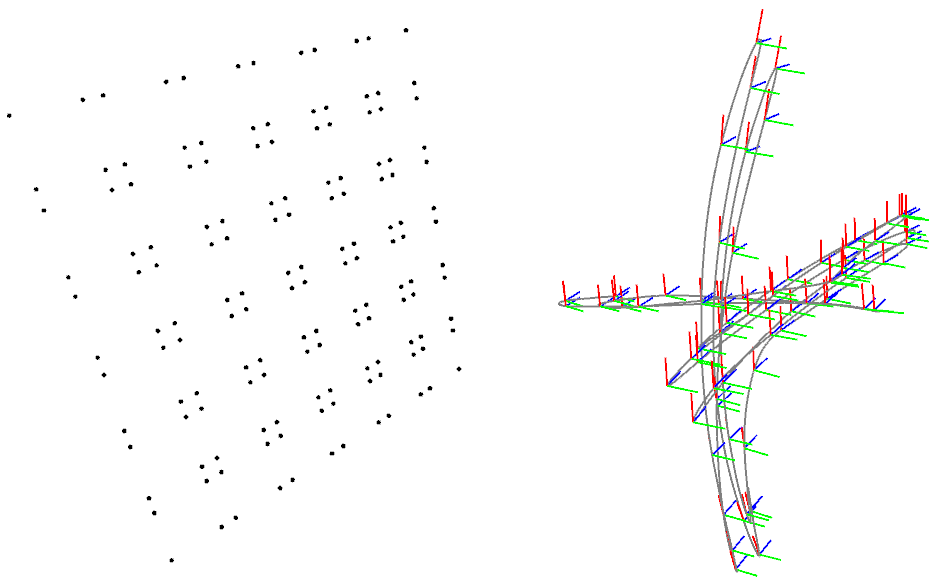} \\ \vspace{2mm}
    \includegraphics[width=\linewidth]{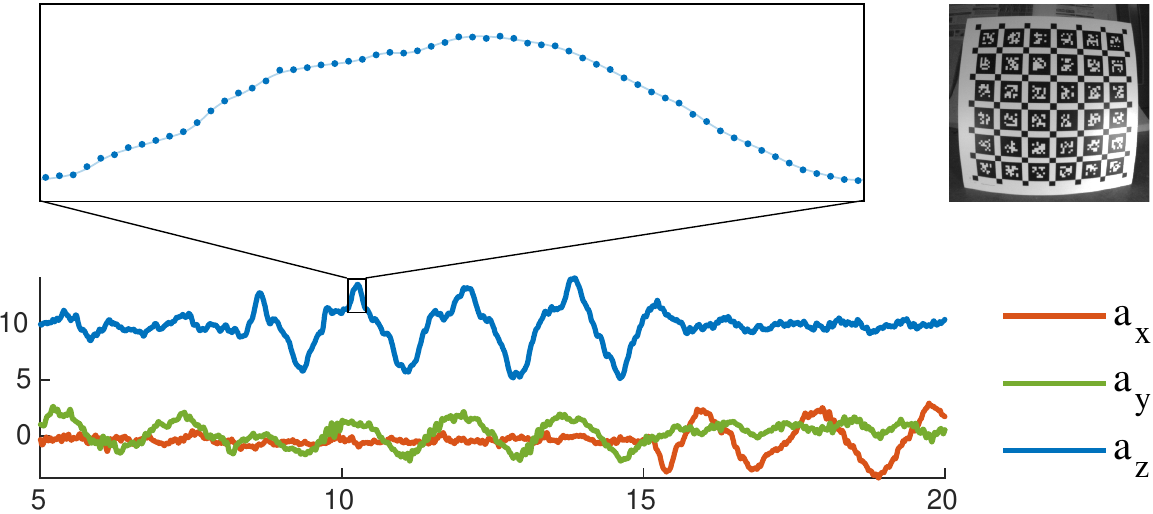}
    \caption{
        Camera-IMU calibration using a Lie group cumulative B-spline to represent the IMU trajectory (gray line with axes used to visualize rotation). Observations of the calibration pattern are combined with accelerometer and gyroscope measurements to estimate the trajectory and calibration parameters in a joint optimization. The plot visualizes the accelerometer measurements in \si{\meter\per\second\squared} (dots) overlaid on the continuous estimate generated from the spline trajectory (line) after optimization. As shown in the experimental section, the proposed formulation is able to significantly reduce the computational effort of such an optimization. 
    }
    \label{fig:traj_pattern}
\end{figure}

\section{Introduction}

Estimating trajectories is a recurring topic in computer vision research:
In odometry and SLAM applications the sensor motion needs to be estimated, in object tracking and robotic grasping tasks, we want to compute the 6DoF pose over time, and for autonomous exploration, path planning and obstacle avoidance, we need to predict good trajectories.
Over the last years, researchers have increasingly switched to {\em continuous-time trajectories}: Instead of a simple list of poses for discrete time points, the trajectory is elegantly represented by a continuous function in time with values in the space of possible poses.
B-splines are a natural choice for parameterizing such functions. They have been used in several well-known works on continuous-time trajectory estimation.
However, since the B-spline trajectories take values in the Lie group of poses, the resulting differential calculus is much more involved than in Euclidean space $\R^\liedim$. Existing approaches to computing time derivatives suffer from a high computational cost that is actually quadratic in the spline order.

In this paper, we introduce recurrence relations for respective time derivatives and show how they can be employed to significantly reduce the computational cost and to derive concrete (analytic) expressions for the spline Jacobians w.r.t.\ the control points.
This is not only of theoretical interest:
by speeding up time derivatives and Jacobian computation significantly, we take a large step towards the real-time capability of continuous-time trajectory representation and its applications such as camera tracking, motion planning, object tracking or rolling-shutter modelling.
In summary, our contributions are the following:

\begin{itemize}
    \item A simple formulation for the time derivatives of Lie group cumulative B-splines that requires a number of matrix operation which scales linearly with the order $k$ of the spline.
    \item Simple (linear in $k$) analytic Jacobians of the value and the time derivatives of an $\SO(3)$ spline with respect to its knots.
    \item Faster optimization time compared to the currently available implementations, due to provably lower complexity. This is demonstrated on simulated experiments and real-world applications such as camera-IMU calibration.
\end{itemize}

\section{Related Work}

This paper consists of two main parts:
first, we take a detailed look at the theory behind B-splines, in particular on Lie groups.
In the second part, we look at possible applications of our efficient derivative computation in computer vision.
In the following, we will review related work for both parts.

\paragraph{B-splines in Lie groups}
Since the 1950s, B-splines have become a popular tool for approximating and interpolating functions of one variable.
Most notably, de Casteljau~\cite{casteljau-59}, Cox~\cite{cox1972numerical} and De~Boor~\cite{de1972calculating} introduced principled ways of deriving spline coefficients from a set of desirable properties, such as locality and smoothness.
Qin~\cite{qin00} found that due to their locality property, B-splines are conveniently expressed using a matrix representation.
By using so-called cumulative B-splines, the concept can be transferred from $\R^d$-valued functions to the more general set of Lie group-valued functions.
This was first done for the group of 3D rotations $\SO(3)$~\cite{kim95}, and later generalized to arbitrary Lie groups $\liegroup$~\cite{sommer2015continuous}.
The latter also contains formulas for computing derivatives of $\liegroup$-valued B-splines, but the way they are formulated is not practical for implementation.
For a general overview of computations in Lie groups and Lie algebras, we refer to~\cite{barfoot2017state, sola2018micro}.

\paragraph{Applications in computer vision}
Thanks to their flexibility in representing functions, B-splines have been used a lot for trajectory representation in computer vision and robotics.
The applications range from calibration~\cite{furgale12, lovegrove2013spline} to odometry estimation with different sensors~\cite{lovegrove2013spline, kerl15, mueggler18}, 3D reconstruction~\cite{ovren2018spline, ovren2019trajectory} and trajectory planning~\cite{usenko2017replanning, ding19}.
All of these works need temporal derivatives of the B-splines at some point, but to the best of our knowledge, there is no work explicitly investigating efficient computation and complexity of these.
Several works have addressed the question if it is better to represent trajectories as one spline in $\SE(3)$, or rather use a split representation of two splines in $\R^3$ and $\SO(3)$~\cite{haarbach2018survey, ovren2018spline, ovren2019trajectory}.
While this cannot be answered unambiguously without looking at the specific use case, all authors come to the conclusion that on average, using the split representation is better both in terms of trajectory representation and in terms of computation time.

\section{Lie Group Foundations}

\subsection{Notation}

A Lie group $\liegroup$ is a group which also is a differentiable manifold, and for which group multiplication and inversion are differentiable operations.
The corresponding Lie algebra $\liealg$ is the tangent space of $\liegroup$ at the identity element $\id$.
Prominent examples of Lie groups are the trivial vector space Lie groups $\R^\liedim$, which have $\liegroup=\liealg=\R^\liedim$ and where the group multiplication is simple vector addition, and matrix Lie groups such as the transformation groups $\SO(n)$ and $\SE(n)$, with matrix multiplication as group multiplication.
Of particular interest in computer vision applications are the groups $\SO(3)$ of 3D rotations and $\SE(3)$, the group of rigid body motions.

The continuous-time trajectories in this paper are functions of time $t$ with values in a Lie group $\liegroup$.
If $\liedim$ denotes the number of degrees of freedom of $\liegroup$, the hat transform $\cdot_\wedge\colon\R^\liedim\to\liealg$ is used to map tangent vectors to elements in the Lie algebra $\liealg$.
The Lie algebra elements can be mapped to their Lie group elements using the matrix exponential $\exp\colon\liealg\to\liegroup$, which has a closed-form expression for $\SO(3)$ and $\SE(3)$. The composition of the hat transform followed by the matrix exponential is given by
\begin{align}
    \Exp\colon\R^d\to\liegroup\,.
\end{align}
Its inverse is denoted
\begin{align}
    \Log\colon\liegroup\to\R^d\,,
\end{align}
which is a composition of the matrix logarithm $\log\colon\liegroup\to\liealg$ followed be the inverse of the hat transform $\cdot_\vee\colon\liealg\to\R^\liedim$.

\begin{definition}
    For an element $\mat{X}\in\liegroup$, the \emph{adjoint} $\Adj{\mat{X}}$ is the linear map defined by
\begin{align}
\label{eq:defadj}
    \Adj{\mat{X}}\vect{v} = (\mat{X}\vect{v}_\wedge \mat{X}^{-1})_\vee \qquad \forall\: \vect{v}\in\R^\liedim\,.
    \end{align}
\end{definition}

It follows readily from the definition in \eqref{eq:defadj} that 
\begin{align}
\label{eq:adj}
    \mat{X}\Exp(\vect{v}) &= \Exp(\Adj{\mat{X}}\vect{v})\mat{X}\quad &\forall\: \vect{v}\in\R^\liedim\,, \\
    \Exp(\vect{v})\mat{X} &= \mat{X}\Exp(\Adj{\mat{X}^{-1}}\vect{v})\quad &\forall\: \vect{v}\in\R^\liedim\,.
\end{align}
For a rotation $\mat{R}\in\SO(3)$, the adjoint is simply $\Adj{\mat{R}}=\mat{R}$.
In this paper, we also use the commutator of two Lie algebra elements:
\begin{definition}
The \emph{commutator} is defined as
\begin{equation}
    \label{eq:comutator}
    \left[\cdot,\cdot\right]\colon \liealg\times\liealg\to\liealg,\quad \left[\mat{V},\mat{W}\right]=\mat{V}\mat{W}-\mat{V}\mat{W}\,.
\end{equation}
\end{definition}
For $\mat{V}=\vect{v}_\wedge, \mat{W}=\vect{w}_\wedge\in\so(3)$, the commutator has the property
\begin{equation}
    \left[\mat{V},\mat{W}\right]_\vee = \vect{v}\times\vect{w}\,.
\end{equation}

\subsection{Differentiation}

To differentiate the trajectories with respect to their parameters, the following definitions and conventions will be used.

\begin{definition}
The \emph{right Jacobian} $\Jr$ is defined by
\begin{equation}
    \Jr(\vect{v})\vect{w} = \lim_{\epsilon\to 0}{\frac{\Log(\Exp(\vect{v})^{-1}\Exp(\vect{v}+\epsilon\vect{w}))}{\epsilon}}
\end{equation}
for all vectors $\vect{w}\in\R^\liedim$.
\end{definition}
Intuitively, the right Jacobian measures how the difference of $\Exp(\vect{v})$ and $\Exp(\vect{v}+\vect{w})$, mapped back to $\R^d$, changes with $\vect{w}$.
It has the following properties:
\begin{align}
    \label{eq:jrprop1}
    &\Log(\Exp(\vect{v})\Exp(\vect{\delta})) 
    = \vect{v} 
    + \Jr(\vect{v})^{-1}\vect{\delta} 
    + \mathcal{O}(\vect{\delta}^2)\,, \\
    \label{eq:jrprop2}
    &\Exp(\vect{v}+\vect{\delta}) 
    = \Exp(\vect{v})\Exp(\Jr(\vect{v})\vect{\delta}) + \mathcal{O}(\vect{\delta}^2)\,.
\end{align}
If $\liegroup=\SO(3)$, the right Jacobian and its inverse can be found in~\cite[p.~40]{chirikjian2011stochastic}.

Whenever an expression $f(\mat{X})$ is differentiated w.r.t.\ to a Lie group element $\mat{X}$, the derivative is defined as
\begin{align}
    \frac{\partial f(\mat{X})}{\partial \mat{X}} = \left. \frac{\partial f(\Exp(\vect{\delta})\mat{X})}{\partial \vect{\delta}}\right|_{\vect{\delta}=0}.
\end{align}
Consequently, an update step for the variable $\mat{X}$ during optimization is performed as $\mat{X}\leftarrow\Exp(\vect{\delta})\mat{X}$, where $\vect{\delta}$ is the the increment determined by the optimization algorithm.

\section{B-Spline Foundations}

\subsection{Basics}

B-splines define a continuous function using a set of control points (knots), see also Fig.~\ref{fig:spline}.
They have a number of desirable properties for continuous trajectory representation, in particular locality and $C^{k-1}$ smoothness for a spline of order $k$ (degree $k-1$).
We will focus on uniform B-splines of order $k$ in this work.
\begin{definition}
A \emph{uniform B-spline} of order $k$ is defined by its control points $\vect{p}_i$ ($0\leq i\leq N$), times $t_i=t_0+i\Delta t$ and a set of spline coefficients $B_{i,k}(t)$:
\begin{equation}
    \label{eq:bspline}
    \vect{p}(t) = \sum_{i=0}^{N}{B_{i,k}(t)\,\vect{p}_i}\,,
\end{equation}
where the coefficients are given by the De~Boor--Cox recurrence relation~\cite{cox1972numerical,de1972calculating}
\begin{align}
B_{i,0}(t) &= 
        \begin{cases}
        1, & \text{for~~} t_{i}\leq t< t_{i+1}\\
        0, & \text{otherwise}
        \end{cases} \\
B_{i,j}(t) &= \frac{t-t_i}{j\Delta t} B_{i,j-1}(t) + \frac{t_{i+j+1}-t}{j\Delta t} B_{i+1,j-1}(t)\,.
\end{align}
\end{definition}
It is possible to transform \eqref{eq:bspline} into a cumulative representation:
\begin{align}
\label{eq:cumulative_spline}
\vect{p}(t) &= \widetilde B_{0,k}(t)\,\vect{p}_0  + \sum_{i=1}^{N}{\widetilde B_{i,k}(t) (\vect{p}_i - \vect{p}_{i-1})}\,, \\
\widetilde B_{i,k}(t) &= \sum_{s=i}^{N} B_{s,k}(t)\,.
\end{align}

\begin{figure}
    \centering
    \includegraphics[width=7.2cm, height=4.2cm]{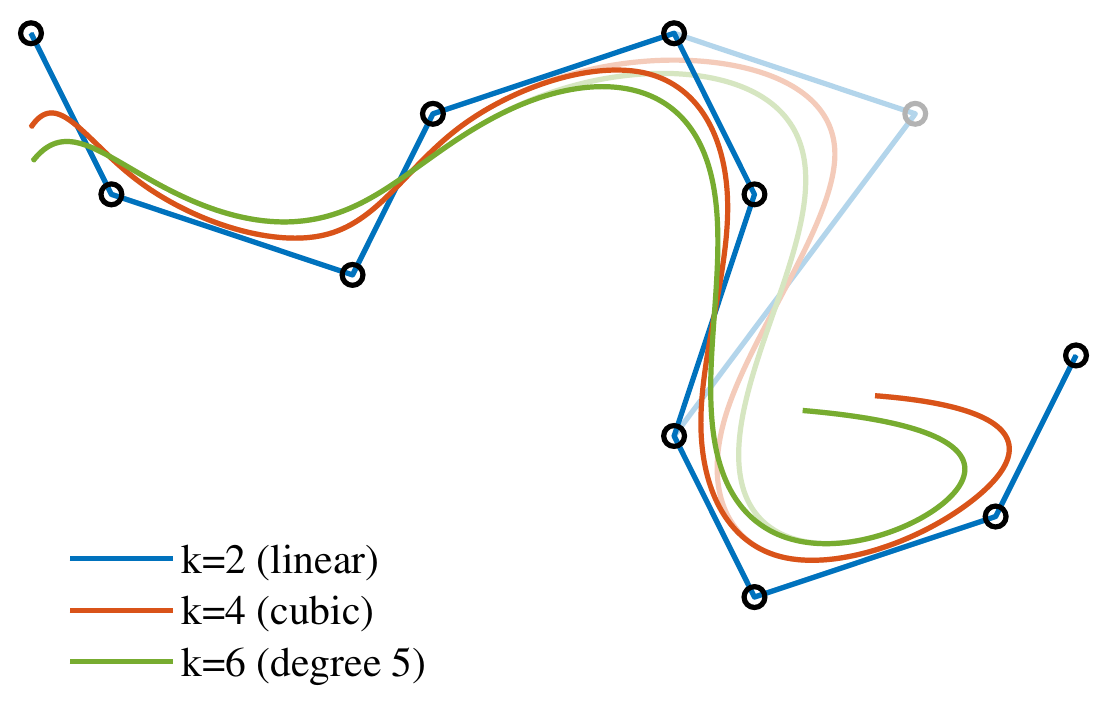}
    \caption{
        A set of control points (black) in $\R^2$, and the resulting B-splines for different orders $k$.
        For the linear spline ($k=2$), the spline curve actually hits the control points, while for higher order splines, this is not true in general.
        The lighter lines show how the splines change if one control point changes:
        the curves only change locally, i.e. in the vicinity of the modified control point.
    }
    \label{fig:spline}
\end{figure}

\subsection{Matrix Representation}

B-splines have local support, which means that for a spline of order $k$, at a given $t$ only $k$ control points contribute to the value of the spline.
As shown in \cite{qin00}, it is possible to represent the spline coefficients using a matrix representation, which is constant for uniform B-splines.

At time $t \in [t_i, t_{i+1})$ the value of $\vect{p}(t)$ only depends on the control points $\vect{p}_{i}$, $\vect{p}_{i+1}$, ..., $\vect{p}_{i+k-1}$.
To simplify calculations, we transform time to a uniform representation $s(t) := (t - t_0)/\Delta t$, such that the control points transform into $\{0,..,k-1\}$.
We define $u(t) := s(t)-i$ as normalized time elapsed since the start of the segment $[t_i, t_{i+1})$ and from now on use $u$ as temporal variable.
The value of $\vect{p}(u)$ can then be evaluated using a matrix representation as follows~\cite{qin00}:
\begin{align}
    \underbrace{\vect{p}(u)}_{\liedim\times 1} &=
    \underbrace{\begin{pmatrix} \vect{p}_{i}, & \vect{p}_{i+1}, & \cdots, & \vect{p}_{i+k-1}
    \end{pmatrix}}_{\liedim\times k} \underbrace{\mat{M}^{(k)}}_{k\times k}
    \underbrace{\vect{u}}_{k\times 1}\,,
\end{align}
where $u_n=u^n$, $\mat{M}^{(k)}$ is a blending matrix with entries
\begin{equation}
\begin{aligned}
    m_{s,n}^{(k)} &= \frac{C^{n}_{k-1}}{(k-1)!} \sum^{k-1}_{l=s} (-1)^{l-s} C^{l-s}_{k} (k-1-l)^{k-1-n}\,, \\
    & s,n \in \{0, \dots, k-1\}\,,
\end{aligned}
\end{equation}
and $C_k^s=\frac{k!}{s!\,(k-s)!}$ are binomial coefficients.

It is also possible to use the matrix representation for the cumulative splines:
\begin{equation}
\label{eq:cum_matrix}
    \vect{p}(u) =
\begin{pmatrix} 
\vect{p}_i\,, & \vect{d}_1^i\,, & \cdots\,, & \vect{d}_{k-1}^i
\end{pmatrix}
\mat{\widetilde M}^{(k)}
\vect{u}\,,
\end{equation}   
with cumulative matrix entries $\widetilde m_{j,n}^{(k)} = \sum_{s=j}^{k-1} m_{s,n}^{(k)}$ and difference vectors $\vect{d}_j^i = \vect{p}_{i+j}-\vect{p}_{i+j-1}$.

We show in the Appendix that the first row of the matrix $\mat{\widetilde M}^{(k)}$ is always equal to the unit vector $\vect{e}_0\in\R^k$:
\begin{equation}
    \label{eq:mtilde}
    \widetilde m_{0,n}^{(k)} = \delta_{n,0} \quad \forall\: n=0,\cdots,k-1\,,
\end{equation}
with the Kronecker delta $\delta_{n,0}$.
In particular, if we define 
\begin{equation}
    \label{eq:lambda}
    \vect{\lambda}(u) = \mat{\widetilde{M}}^{(k)}
    \vect{u}\,,
\end{equation}
this implies $\lambda_0(u)\equiv 1$.
Inserting $\vect{\lambda}$ with $\lambda_0=1$ into the cumulative matrix representation \eqref{eq:cum_matrix} allows us to write $\vect{p}(u)$ conveniently as follows:

\begin{theorem}
    The B-spline of order $k$ at position $u$ can be written as
\begin{equation}
    \label{eq:local_spline}
    \vect{p}(u) = \vect{p}_i + \sum_{j=1}^{k-1}{\lambda_j(u)\cdot\vect{d}^i_j}\,.
\end{equation}
\end{theorem}

\section{Cumulative B-splines in Lie groups}

The cumulative B-spline in \eqref{eq:local_spline} can be generalized to Lie groups~\cite{sommer2015continuous}, and in particular to $\SO(3)$ for smooth rotation generation~\cite{kim95}.
First, a simple $\R^\liedim$-addition in \eqref{eq:local_spline} corresponds to the group multiplication in a general Lie group (matrix multiplication in matrix Lie groups).
Second, while we can easily scale a vector $\vect{d}\in\R^\liedim$ by a factor $\lambda\in\R$ using scalar multiplication in $\R^\liedim$, the concept of scaling does not exist for elements $\mat{X}$ of a Lie group.
Thus, we first have to map $\mat{X}$ from $\liegroup$ to the Lie algebra $\liealg$, which is a vector space, then scale the result, and finally map it back to $\liegroup$: $\Exp(\lambda\cdot\Log(\mat{X}))$.
These two observations together lead to the following definition of cumulative B-splines in Lie groups:
\begin{definition}
    The cumulative B-spline of order $k$ in a Lie group $\liegroup$ with control points $\mat{X}_0,\cdots,\mat{X}_N\in\liegroup$ has the form
    \begin{equation}
        \label{eq:bspline_lie}
        \mat{X}(u) = \mat{X}_i \cdot \prod_{j=1}^{k-1}{\Exp\left(\lambda_j(u)\cdot\vect{d}^i_j\right)}\,,
    \end{equation}
    with the generalized difference vector $\vect{d}_j^i$
    \begin{equation}
        \vect{d}^i_j = \Log\left(\mat{X}_{i+j-1}^{-1}\mat{X}_{i+j}\right) \in \R^d\,.
    \end{equation}
\end{definition}
We should mention that as opposed to a B-spline in $\R^\liedim$, the order of multiplication (addition in \eqref{eq:local_spline}) does matter here, and different generalizations to Lie groups are possible in principle.
In practice, we use the convention that is most commonly used in related work~\cite{lovegrove2013spline, kerl15, mueggler18, ovren2018spline, ovren2019trajectory}.
We omit the $i$ to simplify notation, and define
\begin{align}
    \mat{A}_j(u) &= \Exp\left(\lambda_j(u)\cdot\vect{d}_j\right)
\end{align}
to obtain the more concise expression
\begin{equation}
    \label{eq:conciseR}
    \mat{X}(u) = \mat{X}_i \cdot \prod_{j=1}^{k-1}{\mat{A}_j(u)}\,.
\end{equation}
Note that this can be re-formulated as a recurrence relation:
\begin{align}
    \mat{X}(u) &= \mat{X}^{(k)}(u)\,, \\
    \label{eq:recursive_R}
    \mat{X}^{(j)}(u) &= \mat{X}^{(j-1)}(u)\mat{A}_{j-1}(u) \,, \\
    \mat{X}^{(1)}(u) & = \mat{X}_i\,.
\end{align}

\subsection{Time derivatives}

The main contribution of this paper is a simplified representation of the temporal derivatives, which needs less operations compared to related work.
We first review the derivatives according to the current standard, and then introduce ours.
We denote differentiation w.r.t.\ $u$ by a dot and apply the product rule to \eqref{eq:conciseR} to get
\begin{align}
    \dot{\mat{X}}(u) = \mat{X}_i\cdot\sum_{j=1}^{k-1}{\left(\prod_{l=1}^{j-1}{\mat{A}_l(u)}\right)\dot{\mat{A}}_j(u)\left(\prod_{l=j+1}^{k-1}{\mat{A}_l(u)}\right)}
\end{align}
with
\begin{equation}
    \label{eq:A_dot}
    \dot{\mat{A}}_j(u) = \dot\lambda_j(u)\mat{A}_j(u)\mat{D}_j = \dot\lambda_j(u)\mat{D}_j\mat{A}_j(u)\,,
\end{equation}
and $\mat{D}_j=(\vect{d}_j)_\wedge$. Note that $\mat{A}_j(u)$ and $\mat{D}_j$ commute by definition.
For the case of cubic splines ($k=4$), this reduces to
\begin{equation}
    \dot{\mat{X}} = \mat{X}_i\left(\dot{\mat{A}}_1 \mat{A}_2 \mat{A}_3 + \mat{A}_1 \dot{\mat{A}}_2 \mat{A}_3 + \mat{A}_1 \mat{A}_2 \dot{\mat{A}}_3\right),
\end{equation}
which is the well known formula from e.g.~\cite{lovegrove2013spline, mueggler18, ovren2018spline, ovren2019trajectory}.
An implementation following this formula needs to perform $(k-1)^2+1$ matrix-matrix multiplications and is thus quadratic in the spline degree.
We propose to define the time derivatives recursively instead, which needs less operations:

\begin{theorem}\label{th:td}
    The time derivative $\dot{\mat{X}}$ is given by the following recurrence relation:
    \begin{align}
        \label{eq:Rdot}
        \dot{\mat{X}} &= \mat{X} \vect{\omega}^{(k)}_\wedge\,, \\
        \label{eq:recursion_omega}
        \vect{\omega}^{(j)} &= \Adj{\mat{A}_{j-1}^{-1}} \vect{\omega}^{(j-1)} + \dot \lambda_{j-1}\vect{d}_{j-1} \in \R^d\,, \\
        \vect{\omega}^{(1)} &= \vect{0} \in \R^\liedim\,.
    \end{align}
    $\vect{\omega}^{(k)}$ is commonly referred to as \emph{velocity}.
    For $\liegroup=\SO(n)$, we also call it \emph{angular velocity}.
\end{theorem}
\begin{proof}
    We use the recursive definition of $\mat{X}(u)$ in \eqref{eq:recursive_R} and prove by induction over $j$ that
    \begin{equation}
        \label{eq:Rdot_induction}
        \dot{\mat{X}}^{(j)} = \mat{X}^{(j)}\vect{\omega}^{(j)}_\wedge\,,
    \end{equation}
    which is equivalent to the claim for $j=k$.
    First, we note that for $j=1$, $\mat{X}^{(j)}(u) = \mat{X}_i$ is constant w.r.t.\ $u$, and thus $\dot{\mat{X}}^{(1)} = 0 = \mat{X}^{(1)}\hat{\vect{\omega}}^{(1)}_\wedge$.
    Now, let 
    \eqref{eq:Rdot_induction} be true for some $j\in\{1, \dots, k-1\}$, then we have
    \begin{equation}
    \begin{aligned}
        \dot{\mat{X}}^{(j+1)} &= \partial_u\left(\mat{X}^{(j)}\mat{A}_j\right) = \dot{\mat{X}}^{(j)}\mat{A}_j + \mat{X}^{(j)}\dot{\mat{A}}_j \\
            &= \mat{X}^{(j)}\left(\vect{\omega}^{(j)}_\wedge \mat{A}_j + \dot\lambda_k\mat{A}_j\mat{D}_j \right) \\
            &= \mat{X}^{(j)}\mat{A}_j \left(\mat{A}_j^{-1}\vect{\omega}^{(j)}_\wedge \mat{A}_j + \dot\lambda_j (\vect{d}_j)_\wedge \right) \\
            &= \mat{X}^{(j+1)} \underbrace{\left(\left( \Adj{\mat{A}_j^{-1}}\vect{\omega}^{(j)}\right)_\wedge + \dot\lambda_j(\vect{d}_j)_\wedge\right)}_{=\vect{\omega}^{(j+1)}_\wedge}. \\
    \end{aligned}
    \end{equation}
\end{proof}

Note that our recursive definition of $\dot{\mat{X}}$ makes it very easy to see that $\mat{X}^{-1}\dot{\mat{X}}=\vect{\omega}^{(k)}_\wedge\in\liealg$ for any Lie group, a property which is implicitly used in many works~\cite{lovegrove2013spline, ovren2019trajectory}, but never shown explicitly for arbitrary $\liegroup$.

The scheme presented in Theorem \ref{th:td} computes time derivatives with only $k-1$ matrix-vector multiplications and vector additions, together with one single matrix-matrix multiplication.

Since rotations in $\SO(3)$ are a common and important use case of B-splines in Lie groups, we explicitly state the angular velocity recursion \eqref{eq:recursion_omega} for $\liegroup=\SO(3)$:
\begin{equation}
    \label{eq:recursion_omega_so3}
    \vect{\omega}^{(j)} = \mat{A}_{j-1}^\top \vect{\omega}^{(j-1)} + \dot{\lambda}_{j-1}\vect{d}_{j-1}\,.
\end{equation}

For second order time derivatives, it is easy to see that the calculations proposed in related works~\cite{lovegrove2013spline, mueggler18} need
\begin{equation}
    k(k-1) + k~C_{k-1}^2 = \frac{1}{2}k^2(k-1)
\end{equation}
matrix-matrix multiplications and are thus cubic in the spline order.
We propose a different way to compute $\ddot{\mat{X}}$:

\begin{theorem}
    \label{th:ttd}
    The second derivative of $\mat{X}$ w.r.t.\ $u$ can be computed by the following recurrence relation:
    \begin{equation}
        \label{eq:Rddot}
        \ddot{\mat{X}} = \mat{X}\left((\vect{\omega}^{(k)})_\wedge^2 + \dot{\vect{\omega}}^{(k)}_\wedge\right),
    \end{equation}
    where the \emph{(angular) acceleration} $\dot{\vect{\omega}}^{(k)}$ is recursively defined by
    \begin{align}
    \label{eq:second_t_recursion}
    &\begin{aligned}
        \dot{\vect{\omega}}^{(j)} &= \dot\lambda_{j-1}\left[\vect{\omega}^{(j)}_\wedge, \mat{D}_{j-1}\right]_\vee \\
        &\quad +\Adj{\mat{A}_{j-1}^{-1}}\dot{\vect{\omega}}^{(j-1)} + \ddot{\lambda}_{j-1}\vect{d}_{j-1}\,,
    \end{aligned} \\
    &\dot{\vect{\omega}}^{(1)} = \vect{0}\in\R^\liedim\,.
    \end{align}
\end{theorem}
\begin{proof}
\eqref{eq:Rddot} follows from our expression for $\dot{\mat{X}}$ \eqref{eq:Rdot} and the product rule.
For \eqref{eq:second_t_recursion}, the last summand is trivial, so we focus on the derivative of the first term in the velocity recursion \eqref{eq:recursion_omega}:
first, we note that
\begin{equation}
    \Adj{\mat{A}_{j-1}^{-1}} \vect{\omega}^{(j-1)} = \left(\mat{A}_{j-1}^{-1}\vect{\omega}^{(j-1)}_\wedge \mat{A}_{j-1}\right)_\vee =: \bar{\vect{\omega}}\,,
\end{equation}
so the time derivative of that term is a sum of three terms following the product rule for differentiation.
The middle term is
\begin{equation}
    \left(\mat{A}_{j-1}^{-1}\dot{\vect{\omega}}^{(j-1)}_\wedge \mat{A}_{j-1}\right)_\vee = \Adj{\mat{A}_{j-1}^{-1}} \dot{\vect{\omega}}^{(j-1)}\,,
\end{equation}
which is exactly the second summand in \eqref{eq:second_t_recursion}.
The remaining two terms are 
\begin{equation}
\begin{aligned}
    &\left(\dot{\mat{A}}_{j-1}^{-1}\vect{\omega}^{(j-1)}_\wedge \mat{A}_{j-1} + \mat{A}_{j-1}^{-1}\vect{\omega}^{(j-1)}_\wedge \dot{\mat{A}}_{j-1}\right)_\vee \\
    &\quad \labeleq{(\ref{eq:A_dot})} \left(-\dot{\lambda}_{j-1}\mat{D}_{j-1}\bar{\vect{\omega}}_\wedge + \bar{\vect{\omega}}_\wedge\dot{\lambda}_{j-1}\mat{D}_{j-1}\right)_\vee \\
    &\quad \labeleq{(\ref{eq:comutator})} \dot{\lambda}_{j-1}\left[\bar{\vect{\omega}}_\wedge, \mat{D}_{j-1} \right]_\vee \\
    &\quad \labeleq{(\ref{eq:recursion_omega})} \dot{\lambda}_{j-1}\left[\vect{\omega}^{(j)}_\wedge - \dot{\lambda}_{j-1}\mat{D}_{j-1}, \mat{D}_{j-1} \right]_\vee \\
    &\quad = \dot{\lambda}_{j-1}\left[\vect{\omega}^{(j)}_\wedge, \mat{D}_{j-1} \right]_\vee\,.
\end{aligned}
\end{equation}
\end{proof}

This proposed scheme computes second time derivatives with only $2k$ matrix-matrix multiplications, $k-1$ matrix-vector multiplications, $3(k-1)$ vector additions and one matrix addition in any $\liegroup$.
For $\liegroup=\SO(3)$, the acceleration $\dot{\vect{\omega}}^{(j)}$ in \eqref{eq:second_t_recursion} simplifies to
\begin{equation}
\label{eq:second_t_recursion_so3}
  \dot{\vect{\omega}}^{(j)} = \dot{\lambda}_{j-1}\vect{\omega}^{(j)}\times\vect{d}_{j-1} + \mat{A}_{j-1}^\top \dot{\vect{\omega}}^{(j-1)} + \ddot{\lambda}_{j-1}\vect{d}_{j-1}\,.
\end{equation}
This implies that for $\SO(3)$, second order time derivatives only need $3(k-1)$ matrix-vector multiplications and vector additions plus two matrix-matrix multiplications, reducing computation time even further.

The iterative scheme for the computation of time derivatives can be extended to higher order derivatives.
As an example, we provide third order time derivatives of $\mat{X}$ together with the jerk $\ddot{\vect{\omega}}^{(k)}$ in the Appendix.
The number of matrix operations needed to compute this is still linear in the order of the spline.
We also provide a comprehensive overview of the matrix operations needed for the different approaches in the Appendix.

\subsection{Jacobians w.r.t.\ control points in $\SO(3)$}
\label{sec:jacobians}

To emphasize that this paragraph focuses on $\SO(3)$, we denote elements by $\mat{R}$ instead of $\mat{X}$.
The values of both the spline itself and its velocity and acceleration depend on the choice of control points. 
For the derivatives w.r.t.\ the control points of the B-spline, we first note that a control point $\mat{R}_{i+j}$ appears implicitly in $\vect{d}_j$ and $\vect{d}_{j+1}$, and for $j=0$, we also have the explicit dependence of $\mat{R}(u)$ on $\mat{R}_i$.
Thus, we compute derivatives w.r.t.\ the $\vect{d}_j$ and then apply the chain rule.
We focus on $\liegroup=\SO(3)$ as it is the most relevant group for computer vision applications.

In order to apply the chain rule, we need the derivatives of the $\vect{d}_j$ w.r.t.\ the $\mat{R}_{i+j}$, which follow trivially from the definition of the right Jacobian and the adjoint of $SO(3)$:
\begin{equation}
    \frac{\partial{\vect{d}_j}}{\partial \mat{R}_{i+j}} = \Jr^{-1}(\vect{d}_j)\mat{R}_{i+j}^{\top}\,, \quad
    \frac{\partial{\vect{d}_{j+1}}}{\partial \mat{R}_{i+j}} = -\frac{\partial{\vect{d}_{j+1}}}{\partial \mat{R}_{i+j+1}}\,.
\end{equation}
Now consider a curve $\vect{f}$ that maps to $\R^\liedim$, for example the spline value, velocity or acceleration. $\vect{f}$ depends on the set of control points $\mat{R}_{i+j}$ and has derivatives
\begin{equation}
    \frac{\dif \vect{f}}{\dif \mat{R}_{i+j}} = \ddj{\vect{f}} \cdot \frac{\partial{\vect{d}_j}}{\partial \mat{R}_{i+j}} + \frac{\partial \vect{f}}{\partial\vect{d}_{j+1}} \cdot \frac{\partial{\vect{d}_{j+1}}}{\partial \mat{R}_{i+j}}\,.
\end{equation}
for $j>0$.
For $j=0$ we obtain
\begin{equation}
    \frac{\dif \vect{f}}{\dif \mat{R}_{i}} = \frac{\partial \vect{f}}{\partial\mat{R}_i} + \frac{\partial \vect{f}}{\partial\vect{d}_{1}} \cdot \frac{\partial{\vect{d}_1}}{\partial \mat{R}_{i}}\,.
\end{equation}
Thus, to compute Jacobians w.r.t.\ control points, we need the partial derivatives w.r.t.\ the $\vect{d}_j$ as well as $R_i$.

In the following, we will first derive some useful properties, and then present a recursive scheme to compute Jacobians of $\vect{\rho}$, $\vect{\omega}=\vect{\omega}^{(k)}$ and $\dot{\vect{\omega}}=\dot{\vect{\omega}}^{(k)}$, where we define the vector $\vect{\rho}\in\R^d$ as the mapping of $R$ to $\R^d$ by the $\Log$ map:
\begin{equation}
    \vect{\rho}(u) = \Log \mat{R}(u)\,.
\end{equation}
The Jacobians of $\vect{\omega}^{(j)}$ and $\dot{\vect{\omega}}^{(j)}$ w.r.t.\ $\vect{d}_j$ are zero:
from the recursion schemes of $\vect{\omega}$ and $\dot{\vect{\omega}}$  in Theorems \ref{th:td} and \ref{th:ttd}, we find that the first index for which $\vect{d}_j$ appears explicitly or implicitly (in the form of $A_j$) is $j+1$.

Furthermore, we use the following important relation in our derivations, which is proven in the Appendix:
\begin{equation}
\label{eq:jac_A_omega}
    \frac{\partial}{\partial\vect{d}}\Exp(-\lambda\vect{d})\vect{\omega} = \lambda\Exp(-\lambda\vect{d})\vect{\omega}_\wedge\Jr(-\lambda\vect{d})
\end{equation}
for $\lambda\in\R$ and $\vect{d},\vect{\omega}\in\R^3$.
Together, these findings have two important implications:

\begin{theorem}
The Jacobian of $\vect{\omega}^{(j+1)}$ w.r.t.\ $\vect{d}_j$ is
\begin{equation}
    \ddj{\vect{\omega}^{(j+1)}} = \lambda_j\mat{A}_j^\top\vect{\omega}^{(j)}_\wedge \Jr(-\lambda_j\vect{d}_j) + \dot{\lambda}_j\id\,.
\end{equation}
\end{theorem}
\begin{proof}
We apply \eqref{eq:jac_A_omega} to the angular velocity $\vect{\omega}^{(j+1)}$ as derived for $\liegroup=\SO(3)$ in \eqref{eq:recursion_omega_so3}.
\end{proof}

\begin{theorem}
The Jacobian of $\dot{\vect{\omega}}^{(j+1)}$ w.r.t.\ $\vect{d}_j$ is
\begin{equation}
    \begin{aligned}
    \ddj{\dot{\vect{\omega}}^{(j+1)}} &= \dot{\lambda}_j\left(\vect{\omega}^{(j+1)}_\wedge-\mat{D}_j\ddj{\vect{\omega}^{(j+1)}} \right) \\
    &\qquad+\lambda_j\mat{A}_j^\top\dot{\vect{\omega}}^{(j)}_\wedge\Jr(-\lambda_j\vect{d}_j) + \ddot{\lambda}_j\id\,.
    \end{aligned} 
\end{equation}
\end{theorem}
\begin{proof}
We apply \eqref{eq:jac_A_omega} to $\dot{\vect{\omega}}^{(j+1)}$ as defined in \eqref{eq:second_t_recursion_so3} and use
\begin{equation}
    \vect{\omega}\times\vect{d} = \vect{\omega}_\wedge\vect{d} = -\mat{D}\vect{\omega}
\end{equation}
for $\vect{d},\vect{\omega}\in\R^3$ and $\mat{D}=\vect{d}_\wedge$.
\end{proof}

These results for the Jacobians of $\vect{\omega}^{(j+1)}$ and $\dot{\vect{\omega}}^{(j+1)}$ w.r.t.\ $\vect{d}_j$ can be used to derive Jacobians of $\vect{\omega}$ and $\dot{\vect{\omega}}$ by recursion:

\begin{theorem}
The following recurrence relation (from $j=k-1$ to $j=1$) allows for Jacobian computation of $\vect{\rho}$, $\vect{\omega}$ and $\dot{\vect{\omega}}$ in a linear (w.r.t.\ $k$) number of multiplications and additions:
\begin{align}
    \mat{P}_{k-1} &= \id\,, \\
    \vect{s}_{k-1} &= \vect{0}\,, \\
    \label{eq:jac_rho_d}
    \ddj{\vect{\rho}} &= \lambda_j\Jr^{-1}(\vect{\rho})P_j\Jr(\lambda_j\vect{d}_j)\,, \\
    \label{eq:jac_omega_d}
    \ddj{\vect{\omega}} &= P_j\ddj{\vect{\omega}^{(j+1)}}\,, \\
    \label{eq:jac_omegadot_d}
    \ddj{\dot{\vect{\omega}}} &= P_j\ddj{\dot{\vect{\omega}}^{(j+1)}} - (\vect{s}_j)_\wedge\ddj{\vect{\omega}}\,, \\
    \mat{P}_{j-1} &= \mat{P}_j\mat{A}_j^\top\,, \\
    \vect{s}_{j-1} &= \vect{s}_j + \dot\lambda_j\mat{P}_j\vect{d}_j\,.
\end{align}
$P_j$ and $\vect{s}_j$ are accumulator products and sums, respectively.
\end{theorem}

\begin{proof}[Proof of \eqref{eq:jac_rho_d}]
We write $\mat{R}$ as
\begin{equation}
    \mat{R} = \mat{R}_i\mat{A}_\mathrm{pre}\Exp(\lambda_j\vect{d}_j)\mat{A}_\mathrm{post}\,,
\end{equation}
where $\mat{A}_\mathrm{pre}$ and $\mat{A}_\mathrm{post}$ are implicitly defined by comparison with the generic form of a cumulative Lie group B-spline \eqref{eq:bspline_lie}.
They do not depend on $\vect{d}_j$.
The right Jacobian property \eqref{eq:jrprop2}, combined with the adjoint property, yields
\begin{equation}
\begin{aligned}
    &\mat{R}_i\mat{A}_\mathrm{pre}\Exp(\lambda_j(\vect{d}_j+\vect{\delta}))\mat{A}_\mathrm{post} \\
    &= \mat{R}\Exp(\lambda_j\mat{A}_\mathrm{post}^\top\Jr(\lambda_j\vect{d}_j)\vect{\delta}) + \mathcal{O}(\vect{\delta}^2)
\end{aligned}
\end{equation}
Now, we apply the right Jacobian property \eqref{eq:jrprop1} to obtain
\begin{equation}
\begin{aligned}
    &\Log\left(\mat{R}_i\mat{A}_\mathrm{pre}\Exp(\lambda_j(\vect{d}_j+\vect{\delta}))\mat{A}_\mathrm{post}\right) \\
    & = \vect{\rho}  + \lambda_j\Jr^{-1}(\vect{\rho})\mat{A}_\mathrm{post}^\top\Jr(\lambda_j\vect{d}_j)\vect{\delta} + \mathcal{O}(\vect{\delta}^2)\,.
\end{aligned}
\end{equation}
Differentiation at $\vect{\delta}=0$ and inserting $\mat{A}_\mathrm{post}^\top=P_j$ yields \eqref{eq:jac_rho_d}.
\end{proof}

\begin{proof}[Proof of \eqref{eq:jac_omega_d}]
Since $j\leq k-2$, $A_{k-1}$ does not depend on $\vect{d}_j$, thus
\begin{equation}
\begin{aligned}
    \frac{\partial\vect{\omega}}{\partial\vect{d}_j}
        &= \frac{\partial}{\partial\vect{d}_j}\left(\mat{A}_{k-1}^\top\vect{\omega}^{(k-1)} + \dot\lambda_{k-1}\vect{d}_{k-1}\right) \\
        &= \mat{A}_{k-1}^\top\frac{\partial\vect{\omega}^{(k-1)}}{\partial\vect{d}_j}\,.
\end{aligned}
\end{equation}
Iterative application of this equation leads to the claim.
\end{proof}

\begin{proof}[Proof of \eqref{eq:jac_omegadot_d}]
First, since the case $j=k-1$ is trivial, we can focus on $j\leq k-2$:
for these cases, we find (by insertion into \eqref{eq:second_t_recursion_so3} that
\begin{equation}
\label{eq:jac_omegadot_d_norec}
\begin{aligned}
    &\ddj{\dot{\vect{\omega}}}  \\
    &= \ddj{}\left(-\dot\lambda_{k-1}\mat{D}_{k-1}\vect{\omega} + \mat{A}_{k-1}^\top\dot{\vect{\omega}}^{(k-1)} + \ddot{\lambda}_{k-1}\vect{d}_{k-1}\right) \\
    &= -\dot\lambda_{k-1}\mat{D}_{k-1}\ddj{\vect{\omega}} + \mat{A}_{k-1}^\top\ddj{\dot{\vect{\omega}}^{(k-1)}}\,.
\end{aligned}
\end{equation}
We prove the equivalence of this and \eqref{eq:jac_omegadot_d} by induction in the Appendix.
\end{proof}

\section{Experiments}

To evaluate the proposed formulation for the B-spline time derivatives and our $\SO(3)$ Jacobians, we conduct two experiments. In the first one, simulated velocity and acceleration measurements are used to estimate the trajectory represented by the spline. This allows us to highlight the computational advantages of the proposed formulation. In the second experiment, we demonstrate an example of a real-world application, in particular a multiple camera and IMU calibration. In this case we estimate the continuous trajectory of the IMU, transformations from the cameras to the IMU, accelerometer and gyroscope biases and gravity in the world frame.

In both cases, we implemented as baseline method the time derivative computation from prior work~\cite{lovegrove2013spline, mueggler18}.
Unless stated otherwise, optimizations are done using Ceres \cite{ceres-solver} with the automatic differentiation option.
This option uses dual numbers for computing Jacobians.
In all cases, we use the Levenberg-Marquardt algorithm for optimization with sparse Cholesky decomposition for solving linear systems.
The experiments were conducted on Ubuntu 18.04 with Intel Xeon E5-1620 CPU. We used clang-9 as a compiler with \texttt{-O3 -march=native -DNDEBUG} flags. Even though residual and Jacobian computations are easily parallelizable, in this paper we concentrate on differences between formulations and run all experiments in single-thread configuration.

\ifcvprfinal
We have made the experiments available open-source at:
\begin{scriptsize}
\texttt{https://gitlab.com/tum-vision/lie-spline-experiments}
\end{scriptsize}
\fi

\subsection{Simulated Sequence}

\begin{table}
\centering
\begin{tabular}{ c c c  c  c  c}
\toprule
$\liegroup$ & $k$ & Config. & Ours & Baseline & Speedup \\
\midrule
$\SO(3)$ & 4 & acc. & \bf{0.057} & 0.147 & 2.57 \\
$\SO(3)$ & 4 & vel. & \bf{0.058} & 0.088 & 1.52 \\
$\SO(3)$ & 5 & acc. & \bf{0.081} & 0.280 & 3.45 \\
$\SO(3)$ & 5 & vel. & \bf{0.082} & 0.141 & 1.73 \\
$\SO(3)$ & 6 & acc. & \bf{0.117} & 0.520 & 4.43 \\
$\SO(3)$ & 6 & vel. & \bf{0.111} & 0.217 & 1.95 \\
$\SE(3)$ & 4 & acc. & \bf{0.277} & 0.587 & 2.12 \\
$\SE(3)$ & 4 & vel. & \bf{0.253} & 0.334 & 1.32 \\
$\SE(3)$ & 5 & acc. & \bf{0.445} & 1.196 & 2.69 \\
$\SE(3)$ & 5 & vel. & \bf{0.405} & 0.581 & 1.43 \\
$\SE(3)$ & 6 & acc. & \bf{0.644} & 2.332 & 3.62 \\
$\SE(3)$ & 6 & vel. & \bf{0.590} & 0.936 & 1.59 \\
\bottomrule
\end{tabular}
\vspace{2mm}
\caption{Optimization time in seconds for the proposed and baseline formulations with velocity (\emph{vel.}) and acceleration (\emph{acc.}) measurements, and the speedup achieved by our formulation. In all the experiments both formulations converged to the same result with the same number of iterations.}
\label{tab:synthetic_results}
\end{table}

\begin{table*}
\centering
\begin{tabular}{r | c c c c c c c}
\toprule
Estimated variable & $\vect{g}$ [\si{\meter\per\second\squared}] & $\vect{b}_a$ [\si{\meter\per\second\squared}] & $\vect{b}_g$ [\si{\radian\per\second}] & $\vect{t}_{ic0}$ [\si{\meter}]& $\vect{t}_{ic1}$ [\si{\meter}] & $\mat{R}_{ic0}$ [\si{\radian}] & $\mat{R}_{ic1}$ [\si{\radian}] \\
\midrule
Max deviation &
$6.07 \cdot 10^{-5}$ &
$6.32 \cdot 10^{-5}$ &
$2.14 \cdot 10^{-9}$ &
$6.34 \cdot 10^{-6}$ &
$6.33 \cdot 10^{-6}$ &
$3.51 \cdot 10^{-8}$ &
$3.34 \cdot 10^{-8}$ \\
\bottomrule
\end{tabular}
\vspace{1mm}
\caption{Maximum difference between the mean estimate and the estimates from all calibration methods. For vectors ($\vect{g}$, $\vect{b}_a$, $\vect{b}_g$, $\vect{t}_{ic0}$, $\vect{t}_{ic1}$), the $L_2$ norm is used. For rotational values ($\mat{R}_{ic0}$, $\mat{R}_{ic1}$) the angle norm in radians is used.
The results show that independent of the underlying spline representation ($\SO(3)\times\R^3$ or $\SE(3)$) the calibration converges to the same result.}
\label{tab:calibration_accuracy}
\vspace{-3mm}
\end{table*}

One typical application of B-splines on Lie groups is trajectory estimation from a set of sensor measurements. In our first experiment we assume that we have pose, velocity and acceleration measurements for either $\SO(3)$ or $\SE(3)$ and formulate an optimization problem that is supposed to estimate the values of the spline knots representing the true trajectory. In this case we minimize the sum of squared residuals, where a residual is the difference between the measured and the computed value. 

We use a spline with $100 + k$ knots with 2 second spacing, 25 value measurements and 2020 velocity or acceleration measurements that are uniformly distributed across the spline. The measurements are generated from the ground-truth spline. We initialize the knot values of the splines that will be optimized to perturbed ground truth values, which results in 5 optimization iterations until convergence. Table \ref{tab:synthetic_results} summarizes the results. As expected, the proposed formulation outperforms the baseline formulation in all cases. The time difference is higher for the acceleration measurements, since there the baseline formulation is cubic in the order of spline.

\subsection{Camera-IMU calibration}

In the second experiment we aim to show the applicability of our approach for real-world applications with camera-IMU calibration as an example. We use two types of splines of order 5 and knot spacing of \SI{10}{\milli\second} to represent the continuous motion of the IMU coordinate frame: $\SO(3)\times\R^3$ (split representation) and $\SE(3)$. For both cases we implemented the proposed and the baseline method to compute time derivatives.

We use the \emph{calib-cam1} sequence of \cite{schubert18dataset} in this experiment. It contains 51.8 seconds of data and consists of 10336 accelerometer and the same number of gyroscope measurements, 1036 images for two cameras which observe 291324 corners of the calibration pattern. We assume that the camera intrinsic parameters are pre-calibrated and there is a good initial guess between camera and IMU rotations computed from angular velocities. All optimizations in our experiment have the same initial conditions and noise settings. A segment of the sequence trajectory after optimization is visualized in Figure \ref{fig:traj_pattern}.

The projection residuals are defined as:
\begin{align}
    \vectwo{r}_{p}(u) &= \pi( \mat{T}_{ic}^{-1} \mat{T}_{wi}(u)^{-1} \vect{x}) - \hat{\vectwo{p}}, \\
    \mat{T}_{wi} &= \begin{pmatrix}
    \mat{R}_{wi} & \vect{t}_{wi}\\
    0 & 1
    \end{pmatrix} \in \SE(3),
\end{align}
where $\mat{T}_{wi}(u)$ is the pose of the IMU in the coordinate frame computed from the spline either as a pose directly ($\SE(3)$), or as two separate values for rotation and translation ($\SO(3)\times\R^3$). $\mat{T}_{ic}$ is the transformation from the camera where the corner was observed to the IMU, $\pi$ is a fixed projection function, $\vect{x}$ is the 3D coordinate of the calibration corner and $\hat{\vectwo{p}}$ denotes the 2D coordinates of the corner detected in the image.

The gyroscope and accelerometer residuals are defined as:
\begin{align}
    \vect{r}_{\omega}(u) &= \vect{\omega}(u) - \Tilde{\vect{\omega}} + \vect{b}_g, \\
    \vect{r}_{a}(u) &= \mat{R}_{wi}(u)^{-1}(\ddot{\vect{t}}_{wi}(u) + \vect{g}) - \Tilde{\vect{a}} + \vect{b}_a,
\end{align}
where $\Tilde{\vect{\omega}}$ and $\Tilde{\vect{a}}$ are the measurements, $\vect{b}_g$ and $\vect{b}_a$ are static biases and $\vect{g}$ is the gravity vector in world coordinates. $\mat{R}_{wi}(u)$ is the rotation from IMU to world frame. The definition of  $\vect{\omega}(u)$ and $\ddot{\vect{t}}_{wi}(u)$ depends on the spline representation that we use. For the $\SO(3)\times\R^3$ representation, $\vect{\omega}(u)$ is the angular velocity in the body frame computed as in (\ref{eq:recursion_omega_so3}) and $\ddot{\vect{t}}_{wi}(u)$ is the second derivative of the $\R^3$ spline representing the translation of the IMU in the world frame. For $\SE(3)$, $\vect{\omega}(u)$ is the rotational component of the velocity computed in (\ref{eq:recursion_omega}) and $\ddot{\vect{t}}_{wi}(u)$ is the translation part of the second time derivative of the pose computed in (\ref{eq:Rddot}).
The $\SE(3)$ formulation of these residuals is identical to the one used in \cite{lovegrove2013spline}.

The calibration is done by minimizing a function that combines the residuals for all measurements:
\begin{align} 
    E &= \sum \vect{r}^{\top}_{\omega} W_{\omega} \vect{r}_{\omega} + \sum \vect{r}^{\top}_{a} W_{a} \vect{r}_{a} + \sum \vect{r}^{\top}_{p} W_{p} \vect{r}_{p},
\end{align}
where $W_{\omega}$, $W_{a}$, $W_{p}$ are the weight matrices computed using the sensor noise characteristics.

In all conducted experiments the calibration converged to the same values (see Table \ref{tab:calibration_accuracy}) after 12 iterations, so switching to our formulation does not affect accuracy of the solution. Our results confirm previous reports \cite{haarbach2018survey, ovren2019trajectory} that the $\SE(3)$ spline representation does not introduce any advantages compared to the split representation, but requires more computations.

The timing results are presented in Table \ref{tab:calibration_results}. In all cases we can see the advantage of the proposed formulation for time derivatives. In the case of split representation only the gyroscope residuals are affected, so the difference is relatively small, if Ceres Jacobians are used (6\% less time). For the $\SE(3)$ representation, both gyroscope and accelerometer residuals are affected, since we need to compute the second time derivative for linear acceleration. In this case our formulation results in 36.7\% less computation time. We also present the results with our custom solver that uses split representation and the analytic Jacobians for $\SO(3)$ that we introduced in Section~\ref{sec:jacobians}. It results in a further decrease in the computation time and is able to perform the calibration in less than 6 seconds (2.6 times faster than the baseline approach with split representation).

The results indicate that our formulation of the time derivatives requires less computations, especially if second time derivatives need to be computed. This can have an even larger effect for the calibration of multiple IMUs \cite{rehder16}, where even for the split formulation, evaluation of the rotational acceleration is required.

\begin{table}
\centering
\begin{tabular}{c c c c c}
\toprule
\multicolumn{3}{c}{$\SO(3)\times\R^3$} & \multicolumn{2}{c}{$\SE(3)$} \\
\midrule
Ours & Ours & Baseline & Ours & Baseline \\
Analytic & Ceres & Ceres & Ceres & Ceres \\
\midrule
5.82 & 14.18 & 15.09 & 23.56 & 37.14 \\
\bottomrule
\end{tabular}
\vspace{1mm}
\caption{Time in seconds to perform the camera-IMU calibration. Analytic uses a custom solver with the analytic Jacobians for all residuals. All other methods use Ceres solver with dual numbers for Jacobian computations.}
\label{tab:calibration_results}
\vspace{-2mm}
\end{table}

\section{Conclusions}

In this work, we showed how commonly used B-splines on Lie groups can be differentiated (w.r.t.\ time and control points) in a very efficient way.
Both our temporal derivatives and our Jacobians can be computed in $\mathcal{O}(k)$ matrix operations, while traditional computation schemes are up to cubic in $k$.
We mathematically prove the correctness of our statements.
While our contribution has a clear focus on theory, we also demonstrate how choosing a good representation together with our derivative computation lead to speedups of up to 4.4x in practical computer vision applications.
This makes our proposed method highly relevant for real-time applications of continuous-time trajectory representations.


\clearpage

{\small
\bibliographystyle{ieee_fullname}
\bibliography{egbib}
}

\clearpage


\begin{center}
  {\Large \bf Supplemental Material: \\
Efficient Derivative Computation for Cumulative B-Splines on Lie Groups \par}
  \ifcvprfinal \else
  \vspace*{24pt}
  {
  \large
  \lineskip .5em
  \begin{tabular}[t]{c}
     Anonymous CVPR submission\\
     \vspace*{1pt}\\
Paper ID \cvprPaperID 
  \end{tabular}
  \par
  }
  \vskip .5em
  \fi
  
  \vspace*{12pt}

\end{center}

\section{Appendix}

\subsection{Used symbols}

In general, we use lowercase bold symbols for vectors in $\R^\liedim$, uppercase regular symbols for elements in the Lie group $\liegroup$ and the Lie algebra $\liealg$, and lowercase regular symbols for scalars, subscripts and superscripts.

\subsubsection{Symbols with specific meaning}

Some symbols appear repeatedly throughout the paper and have a dedicated meaning, most important:

\begingroup
\renewcommand{\arraystretch}{1.5}
\begin{tabular}{r l}
    $\vect{p}(u), \vect{p}_i$ & spline value and control points in $\R^\liedim$ \\
    $\mat{X}(u), \mat{X}_i$ & spline value and control points in $\liegroup$ \\ 
    $\mat{R}(u), \mat{R}_i$ & spline value and control points in $\SO(3)$ \\
    $\vect{\omega}$ & velocity \\
    $\vect{\delta}$ & small increment \\
\end{tabular}
\endgroup

\subsubsection{Indices}

While some of the used subscripts and superscripts are just dummy indices, others have a specific meaning that does not change throughout the paper:

\begingroup
\renewcommand{\arraystretch}{1.5}
\begin{tabular}{r p{5.5cm}}
    $k$ & order of the spline \\
    $i$ & index of control points/time intervals \\
    $j$ & ranges from $0$ or $1$ to $k-1$, recursion index \\
    $l, m, n, s$ & dummy indices without particular meaning, used for definitions and proofs \\
\end{tabular}
\endgroup

\subsection{Right Jacobian for $\SO(3)$}

If $\liegroup=\SO(3)$, the right Jacobian and its inverse can be found in~\cite[p.~40]{chirikjian2011stochastic}:
\begin{align}
    \Jr(\vect{x}) = \id - \frac{1-\cos(\norm{\vect{x}})}{\norm{\vect{x}}^2}\vect{x}_\wedge + \frac{\norm{\vect{x}}-\sin(\norm{\vect{x}})}{\norm{\vect{x}}^3}\vect{x}_\wedge^2\,,\\
    \Jr(\vect{x})^{-1} = \id + \frac{1}{2}\vect{x}_\wedge + \left(\frac{1}{\norm{\vect{x}}^2} - \frac{1+\cos(\norm{\vect{x}})}{2\norm{\vect{x}}\sin(\norm{\vect{x}})}\right)\vect{x}_\wedge^2\,.
\end{align}

\subsection{Third order time derivatives}

For completeness, we state the third order time derivatives for a general Lie group $\liegroup$ here.
The proofs are in analogy to those of the first and second order time derivatives, thus we do not repeat them here.
\begin{align}
    &\dddot{\mat{X}} = \mat{X}\left((\vect{\omega}^{(k)}_\wedge)^3 + 2\vect{\omega}^{(k)}_\wedge\dot{\vect{\omega}}^{(k)}_\wedge+\dot{\vect{\omega}}^{(k)}_\wedge\vect{\omega}^{(k)}_\wedge+\ddot{\vect{\omega}}^{(k)}_\wedge\right)\,, \\
    \label{eq:jerk}
    &\begin{aligned}
        &\ddot{\vect{\omega}}^{(j)} = \Adj{\mat{A}_{j-1}^{-1}}\ddot{\vect{\omega}}^{(j-1)} + \dddot{\lambda}_{j-1}\vect{d}_{j-1} \\
        &+ \left[\ddot{\lambda}_{j-1}\vect{\omega}^{(j)}_\wedge + 2\dot{\lambda}_{j-1}\dot{\vect{\omega}}^{(j)}_\wedge - \dot{\lambda}_{j-1}^2[\vect{\omega}^{(j)}_\wedge, \mat{D}_{j-1}] , \mat{D}_{j-1}\right]_\vee\,,
    \end{aligned} \\
    &\ddot{\vect{\omega}}^{(1)} = \vect{0}\in\R^\liedim\,.
\end{align}
$\ddot{\vect{\omega}}$ is called \emph{jerk}.
For $\liegroup=\SO(3)$, the expression \eqref{eq:jerk} becomes slightly simpler:
\begin{equation}
    \begin{aligned}
        &\ddot{\vect{\omega}}^{(j)} = \mat{A}_{j-1}^\top\ddot{\vect{\omega}}^{(j-1)} + \dddot{\lambda}_{j-1}\vect{d}_{j-1} \\
        &+ \left(\ddot{\lambda}_{j-1}\vect{\omega}^{(j)} + 2\dot{\lambda}_{j-1}\dot{\vect{\omega}}^{(j)} - \dot{\lambda}_{j-1}^2\vect{\omega}^{(j)}\times \vect{d}_{j-1}\right)\times\vect{d}_{j-1}\,.
    \end{aligned}
\end{equation}

\subsection{Complexity analysis}

While we are not the first to write down temporal derivatives of Lie group splines, we actually are the first to compute these in only $\mathcal{O}(k)$ matrix operations (multiplications and additions).
Additionally, to the best of our knowledge, we are the first to explicitly propose a scheme for Jacobian computation in $\SO(3)$, which also does not need more than $\mathcal{O}(k)$ matrix operations.
In Table~\ref{tab:complexities_du}, we provide an overview of the needed number of multiplications and additions for the temporal derivatives (both in related work and according to the proposed formulation).

\begin{table*}
\centering
\begin{tabular}{c c c c c c c c c c c}
\toprule
& $\dot{\mat{X}}$ Baseline && $\dot{\mat{X}}$ Ours && $\ddot{\mat{X}}$ Baseline && $\ddot{\mat{X}}$ Ours, any $\liegroup$ && $\ddot{\mat{X}}$ Ours, $\SO(3)$ & \\
\midrule
m-m mult. & $(k-1)^2+1$ & \blue{10} & 1 & \blue{1} & $\frac{1}{2}k^2(k-1)$ & \blue{24} & $2k$ & \blue{8} & $2$ & \blue{2} \\
m-v mult. & $0$ & \blue{0} & $k-1$ & \blue{3} & $0$ & \blue{0} & $k-1$ & \blue{3} & $2(k-1)$ & \blue{6} \\
add. & $k-2$ & \blue{2} & $k-1$ & \blue{3} & $\frac{1}{2}k^2(k-1)$ & \blue{24} & $3k-2$ & \blue{10} & $2k-1$ & \blue{7} \\
\bottomrule
\end{tabular}
\vspace{2mm}
\caption{Number of matrix operations needed to compute temporal derivatives of the $\liegroup$-valued splines: \emph{m-m/m-v mult.} denote matrix-matrix and matrix-vector multiplications, respectively. \emph{add.} denotes additions of matrices or vectors. Our formulation needs consistently less operations than the baseline approach.
The \blue{blue numbers} give the number of operations for a cubic spline ($k=4$).}
\label{tab:complexities_du}
\end{table*}

\subsection{Proofs}

\subsubsection{Proof of \eqref{eq:mtilde} (cumulative blending matrix)}

After substituting summation indices $s\leftarrow k-1-s$ and $l\leftarrow l-s$ we get
\begin{equation}
    \widetilde m_{0,n}^{(k)} = \frac{C_{k-1}^n}{(k-1)!}\sum_{s=0}^{k-1}{\sum_{l=0}^{s}{(-1)^lC_k^l(s-l)^{k-1-n}}}\,.
\end{equation}
We now show by induction over $k$ that $\widetilde m_{0,n}^{(k)}=\delta_{n,0}$ for all $n=0,...,k-1$:
for $k=1$, $\widetilde m_{0,n}^{(k)} = 1 = \delta_{0,0}$ is trivial.
Now, assume $\widetilde m_{0,n}^{(k)}=\delta_{n,0}$ for some $k$.

Starting from the induction assumption
\begin{equation}
    \label{eq:binomial_induction}
    \sum_{s=0}^{k-1}{\sum_{l=0}^{s}{(-1)^lC_k^l(s-l)^{k-1-n}}} = (k-1)!\,\delta_{n,0}
\end{equation}
we now show that $\widetilde m_{0,n}^{(k+1)}=\delta_{n,0}$ for $n=0,...,k-1$.
If not indicated otherwise, we use well-known binomial sum properties, as summarized in e.g.~\cite[0.15]{jeffrey2007table}.
As a first step, we use the property $C_{k+1}^l=C_k^l+C_k^{l-1}$ and split the terms in the double sum to obtain
\begin{equation}
    \sum_{s=0}^k{\sum_{l=0}^{s}{(-1)^lC_{k+1}^l(s-l)^{k-n}}} = T_1 + T_2 + T_3 + T_4\,,
\end{equation}
with
\begin{align}
    \label{eq:T1}
  & T_1 = \sum_{s=0}^{k-1}{\sum_{l=0}^{s}{(-1)^lC_k^l(s-l)^{k-n}}}\,, \\
  \label{eq:T2}
  & T_2 = \sum_{l=0}^{k}{(-1)^lC_k^l(k-l)^{k-n}}\,, \\
  \label{eq:T3}
  & T_3 = \sum_{s=0}^{k}{\sum_{l=0}^{s-1}{(-1)^lC_k^{l-1}(s-l)^{k-n}}}\,, \\
  \label{eq:T4}
  & T_4 = \sum_{s=0}^{k}{(-1)^sC_k^{s-1}(s-s)^{k-n}}\,.
\end{align}
It is easy to see from \eqref{eq:binomial_induction} that $T_1 = (k-1)!\,\delta_{n,1}$.
Furthermore, $T_2=k!\,\delta_{n,0}$.
For $T_4$, we have
\begin{equation}
\begin{aligned}
    T_4 &= 0^{k-n}\sum_{s=0}^k{(-1)^sC_k^{s-1}} = \delta_{n,k}\sum_{s=0}^{k-1}{(-1)^{s+1}C_k^s} \\
        &= -\delta_{n,k}(-1)^{k-1}C_{k-1}^{k-1} = \delta_{n,k}(-1)^k\,.
\end{aligned}
\end{equation}
Finally, we need $T_3$, which is the most complicated term:
\begin{equation}
\begin{aligned}
    T_3 &= \sum_{s=0}^{k}{\sum_{l=0}^{s-1}{(-1)^lC_k^{l-1}(s-l)^{k-n}}} \\
        &= \sum_{s=0}^{k-1}{\sum_{l=0}^{s-1}{(-1)^{l+1}C_k^l(s-l)^{k-n}}} \\
        &= -\sum_{s=0}^{k-1}{\sum_{l=0}^s{(-1)^lC_k^l(s-l)^{k-n}}} + \sum_{s=0}^{k-1}C_k^s0^{k-n} \\
        &= -(k-1)!\,\delta_{n,1} + \delta_{n,k}(-1)^{k-1}\,,
\end{aligned}
\end{equation}
where the first equality comes from index shifting ($s\leftarrow s-1$ and $l\leftarrow l-1$), and the last one uses the induction assumption.
In total, we obtain:
\begin{equation}
\begin{aligned}
    &T_1 + T_2 + T_3 + T_4 \\
    &= (k-1)!\,\delta_{n,1} + k!\,\delta_{n,0} \\
    &\quad-(k-1)!\,\delta_{n,1} - \delta_{n,k}(-1)^k + \delta_{n,k}(-1)^k \\
    &= k!\,\delta_{n,0}\,,
\end{aligned}
\end{equation}
which is equivalent to $\widetilde m_{0,n}^{(k+1)}=\delta_{n,0}$ for $n=0,...,k-1$ by definition of $\widetilde m_{j,n}^{(k+1)}$.

What remains is the case $n=k$:
\begin{equation}
    \sum_{s=0}^k{\sum_{l=0}^s{(-1)^lC_{k+1}^l(s-l)^0}} = \sum_{s=0}^k{(-1)^sC_k^s} = 0\,,
\end{equation}
which concludes the proof.

\subsubsection{Proof of \eqref{eq:jac_A_omega} (Jacobian of $\Exp(-\lambda\vect{d})$ multiplied by a vector)}

\begin{equation}
\begin{aligned}
    &\frac{\partial}{\partial \vect{d}}\Exp(-\lambda\vect{d})\vect{\omega}
        = \left. \frac{\partial}{\partial \vect{\delta}}\Exp(-\lambda(\vect{d}+\vect{\delta})) \vect{\omega} \right|_{\vect{\delta}=0}\\
    &= \left .\frac{\partial}{\partial \vect{\delta}}\left(\Exp(-\lambda\vect{d})\Exp(-\Jr(-\lambda\vect{d})\lambda\vect{\delta})\vect{\omega}+\mathcal{O}(\vect{\delta}^2)\right) \right|_{\vect{\delta}=0}\\
    &= -\lambda\Exp(-\lambda\vect{d}) \left. \frac{\partial}{\partial \vect{\delta}}(\Exp(\vect{\delta})\vect{\omega}) \right|_{\vect{\delta}=0} \Jr(-\lambda\vect{d})\\
    &= \lambda\Exp(-\lambda\vect{d})\vect{\omega}_\wedge \Jr(-\lambda\vect{d})\,.\\
\end{aligned}
\end{equation}
For the second equality we have used the right Jacobian property \eqref{eq:jrprop2}. To obtain the last equality, note that
\begin{equation}
    \left.\frac{\partial \Exp(\vect{\delta})\vect{\omega}}{\partial \vect{\delta}}\right|_{\vect{\delta}=0}=-\vect{\omega}_\wedge\,.
\end{equation}

\subsubsection{Proof of \eqref{eq:jac_omegadot_d} (Jacobian of acceleration)}

We show by induction that the following two formulas are equivalent for $l=j+2, \dots, k$:
\begin{align}
    \label{eq:domegadot_dj1}
    \ddj{\dot{\vect{\omega}}^{(l)}} &= -\dot\lambda_{l-1}\mat{D}_{l-1}\ddj{\vect{\omega}^{(l)}} + \mat{A}_{l-1}^\top\ddj{\dot{\vect{\omega}}^{(l-1)}}\,, \\
    \label{eq:domegadot_dj2}
    \ddj{\dot{\vect{\omega}}^{(l)}} &= \mat{P}_j^{(l)} \ddj{\dot{\vect{\omega}}^{(j+1)}} - (\vect{s}^{(l)}_j)_\wedge\ddj{\vect{\omega}^{(l)}}\,,
\end{align}
where we define $\mat{P}_j^{(l)}$ and $\vect{s}_j^{(l)}$ as
\begin{align}
    \mat{P}_j^{(l)} = \left(\prod_{m=j+1}^{l-1}{A_m}\right)^\top\quad &\Rightarrow \quad \mat{P}_j^{(k)} = \mat{P}_j\,, \\
    \vect{s}_j^{(l)} = \sum_{m=j+1}^{l-1}{\dot{\lambda}_m\mat{P}_m\vect{d}_m} \quad &\Rightarrow \quad \vect{s}_j^{(k)} = \vect{s}_j\,.
\end{align}
The case $l=k$ then is the desired results.
For $l=j+2$, we easily see that
\begin{align}
    -\dot{\lambda}_{l-1}\mat{D}_{l-1} &= -(\dot{\lambda}_{j+1}\vect{d}_{j+1})_\wedge = -(\vect{s}_j^{(l)})_\wedge\,, \\
    \mat{A}_{l-1}^\top &= \mat{A}_{j+1}^\top = P_j^{(j+2)} = P_j^{(l)}\,,
\end{align}
which together implies the equivalence of \eqref{eq:domegadot_dj1} and \eqref{eq:domegadot_dj2}.
Now, assume the equivalence holds for some $l\in\{j+2,\dots,k-1\}$ and note that
\begin{align}
    \mat{A}_l^\top \mat{P}_j^{(l)} &= \mat{P}_j^{(l+1)}\,, \\
    \mat{A}_{l}^\top(s_j^{(l)})_\wedge &= \left((s_j^{(l+1)})_\wedge - \dot\lambda_{l}\mat{D}_{l}\right)A_{l}^\top\,.
\end{align}
Then, starting from \eqref{eq:domegadot_dj1} and using the induction assumption, we obtain
\begin{equation}
\begin{aligned}
    \ddj{\dot{\vect{\omega}}^{(l+1)}} &= -\dot\lambda_{l}\mat{D}_{l}\ddj{\vect{\omega}^{(l+1)}} \\
    &\quad+ \mat{A}_{l}^\top\left(P_j^{(l)}\ddj{\dot{\vect{\omega}}^{(j+1)}} - (\vect{s}_j^{(l)})_\wedge\ddj{\vect{\omega}^{(l)}}\right) \\
    &= -\dot\lambda_{l}D_{l}\ddj{\vect{\omega}^{(l+1)}} + P_j^{(l+1)}\ddj{\dot{\vect{\omega}}^{(j+1)}} \\
    &\quad - \left((s_j^{(l+1)})_\wedge - \dot\lambda_{l}\mat{D}_{l}\right)A_l^\top\ddj{\vect{\omega}^{(l)}}\,.
\end{aligned}
\end{equation}
The first and the last summand cancel, and what remains is \eqref{eq:domegadot_dj2} for $l+1$, which concludes the proof.

\end{document}